\documentclass[]{fairmeta}
\usepackage{lineno}
\usepackage{calc}
\usepackage{amsmath}
\usepackage{caption}
\usepackage{multirow}
\usepackage{amsfonts}
\usepackage{xspace}
\usepackage{booktabs}
\usepackage{pdfpages}
\usepackage{xcolor}
\usepackage{tabularx}
\usepackage{bbm}
\usepackage{graphicx}
\usepackage{algorithm}
\usepackage{algorithmic}
\usepackage{enumitem}
\usepackage{pifont}

\usepackage{amsmath}
\usepackage{amssymb}
\usepackage{mathtools}
\usepackage{amsthm}
\usepackage{xspace}
\usepackage{algorithm}
\usepackage{algorithmic}
\usepackage{wrapfig}
\usepackage{enumitem}

\usepackage[export]{adjustbox}

\newcommand{\ours}{\textsc{SAGE-OPD}\xspace}

\newcommand{\tabref}[1]{Table~\ref{#1}}
\newcommand{\figref}[1]{Fig.~\ref{#1}}
\newcommand{\eqnref}[1]{\text{Eq.}~(\ref{#1})}
\newcommand{\secref}[1]{\S\ref{#1}}
\newcommand{\appref}[1]{Appendix~\ref{#1}}

\theoremstyle{plain}
\newtheorem{theorem}{Theorem}[section]
\newtheorem{proposition}[theorem]{Proposition}

\theoremstyle{definition}

\theoremstyle{remark}

\title{\ours: Selective Agent-Guided Intervention for Multi-Turn On-Policy Distillation}

\author{Yuhang Zhou}
\author{Lizhu Zhang}
\author{Yifan Wu}
\author{Mingyi Wang}
\author{Bo Peng}
\author{Jiayi Liu}
\author{Xiangjun Fan}
\author{Zhuokai Zhao}

\affiliation{Meta AI}

\abstract{On-policy distillation (OPD) improves student models by training them on trajectories induced by their own policy, making it a promising approach for mitigating exposure bias in agent training.
However, most OPD studies focus on single-turn settings, while realistic LLM agents interact with environments over multiple turns.
In this regime, early errors can alter future observations and compound across the trajectory, and standard dense token-level OPD becomes brittle, as it may over-penalize semantically valid alternatives, reinforce local degeneracies such as repeated actions, and propagate unreliable teacher supervision on off-distribution histories.
We propose \textbf{\ours}, a verifier-free selective intervention framework specifically designed for multi-turn OPD.
Instead of applying teacher supervision uniformly across all turns, \ours first observes environment feedback and uses teacher judgment to decide whether each student response should be skipped or intervened on.
To further address compounding errors, \ours weights token-level distillation by teacher confidence, reducing the influence of uncertain teacher distributions on corrupted or ambiguous histories.
Finally, \ours applies loss normalization to preserve the overall loss scale of standard OPD while retaining selective turn-level weighting.
Experiments on agent tasks show that \ours consistently improves over baselines, achieving up to a $13.3\%$ relative improvement in ALFWorld unseen success rate over standard OPD.
Ablation studies further demonstrate that turn-level intervention, teacher confidence weighting, and loss normalization provide complementary benefits.
Our results suggest that effective multi-turn OPD should remain on-policy, but teacher supervision should be selectively allocated to turns where intervention is necessary and reliable.
}

\date{\today}
\correspondence{Yuhang Zhou and Zhuokai Zhao at \email{\{zyhang, zhuokai\}@meta.com}\\}

\begin{document}

\maketitle

\section{Introduction}
\label{sec:introduction}

On-policy distillation (OPD) has become a leading approach for post-training open-weight reasoning models~\citep{agarwal2024policy, lu2025onpolicydistillation, song2026survey, zhou2026omniopd}.
By training the student on samples generated from its own policy, OPD mitigates the exposure-bias limitation of offline supervised fine-tuning (SFT), where the student is optimized only on teacher-generated states~\citep{zhou2025mixture, hao2026self}.
This on-policy property becomes especially important when moving from single-turn reasoning to multi-turn LLM agents.
Unlike single-turn generation, a multi-turn agent repeatedly acts in an environment and conditions on the resulting observations, so each student action can influence the future states encountered during training and evaluation.
OPD appears to be a natural fit for agent training, because it allows a strong teacher to supervise the student on histories induced by the student's own behavior, rather than only on expert-generated trajectories.

However, most existing OPD studies primarily evaluate single-turn generation on reasoning tasks~\citep{yang2026learning, jin2026entropy}, where each example is completed in one model rollout and there is no environment state that evolves in response to intermediate actions.
This setting abstracts away a central difficulty of realistic agent learning.
In multi-turn interaction, errors can compound over time (compounding errors): an early incorrect action may alter the environment state, distort later observations, and make subsequent decisions increasingly harder to recover.

Despite the advantage of using student-induced trajectories, recent analyses have revealed that the token-level logit signal underlying standard OPD is brittle~\citep{li2026rethinking, luo2026demystifying}.
This brittleness can be amplified in multi-turn agent training.
First, standard OPD applies dense token-level alignment between the student and teacher distributions, even when several continuations are semantically valid~\citep{li2026rethinking}.
For example, two actions may follow different surface forms or intermediate reasoning paths while still making equivalent progress toward the task.
In such cases, token-level distillation may over-penalize acceptable student behavior simply because it does not match the teacher's preferred wording or local decomposition.
Second, dense token-level gradients can reinforce local degeneracies~\citep{luo2026demystifying}, such as repeated thoughts, repeated actions, or shallow formatting patterns, when these prefixes accidentally receive favorable teacher likelihood.
These local degeneracies are particularly harmful in multi-turn interaction, because a repeated or malformed action does not only affect the current response but also changes the subsequent environment feedback.
Third, compounding errors can make later student-induced histories increasingly off-distribution for the teacher~\citep{li2025beyond}.
The teacher may still provide a distribution over next tokens, but that distribution can become less reliable when the current history is the result of earlier student mistakes.
Therefore, uniformly applying OPD to every token at every turn can both over-correct acceptable behavior and propagate unreliable supervision in corrupted histories.

These observations motivate our central question:
\begin{quote}
    \centering
    \textit{Can multi-turn OPD preserve the benefits of on-policy teacher supervision while selectively applying it only to turns where intervention is necessary and reliable?}
\end{quote}

We answer this question with \textbf{\ours{}}, a verifier-free selective intervention framework for multi-turn OPD.
Our key observation is that multi-turn interaction already segments the student trajectory into turn-level units, providing a natural granularity for teacher critique beyond token-level alignment.
For each turn, \ours{} observes the environment feedback and uses teacher judgment to decide whether the current student response should be skipped or intervened on.
This allows \ours{} to preserve the on-policy benefit of OPD while avoiding unnecessary dense supervision on acceptable or semantically flexible student responses.

However, deciding whether to intervene is not sufficient, since compounding errors may create histories where the teacher is uncertain about the precise token-level correction.
%
%
To reduce the influence of uncertain token-level teacher supervision on student-induced histories, \ours{} further weights each turn by the teacher's predictive confidence.
The final selective weight is the product of the intervention weight and the confidence weight, so strong OPD supervision is applied mainly when intervention is needed and the teacher signal is reliable.
Finally, we apply loss normalization to preserve the overall loss scale of standard OPD while retaining the relative turn-level weighting.

We evaluate \ours on various multi-turn agent benchmarks.
Notably, \ours achieves up to a $13.3\%$ relative improvement in success rate over standard OPD with Qwen-3 models~\citep{yang2025qwen3} on ALFWorld~\citep{shridhar2020alfworld}.
Ablation studies further show that turn-level intervention, teacher confidence weighting, and loss normalization provide complementary benefits.

Our contributions are summarized as follows:
\begin{itemize}
    \item We propose \ours{}, a verifier-free selective intervention framework for multi-turn OPD that preserves on-policy teacher supervision while avoiding uniformly dense distillation over all student-generated turns.
    \item We introduce three key components that make selective multi-turn OPD effective: teacher-judged turn-level intervention for deciding whether supervision is needed, teacher confidence weighting for estimating whether token-level supervision is reliable, and loss normalization for preserving the training signal scale.
    \item We demonstrate the effectiveness of \ours{} across embodied interaction, scientific task solving, and question answering, with strong gains over standard OPD, and other state-of-the-art OPD baselines.
\end{itemize}
\section{Related Work}
\label{sec:related_work}

\paragraph{Multi-turn LLM Agents.}
Recent LLM agents extend language models from single-turn text generation to interactive decision making, where the model repeatedly observes an environment, produces an action or tool call, and receives new feedback~\citep{yao2022react, yang2024swe, zhou2025mixture, chan2025mle, chen2025scaling, zhou2026synthetic}.
These multi-turn agents have been studied in embodied simulators~\citep{shridhar2020alfworld}, web navigation~\citep{gur2024real, team2025tongyi}, tool use~\citep{yao2024tau}, and search-based question answering~\citep{jin2025search}. 
Compared with single-turn tasks, multi-turn agent learning is more challenging because the model must maintain a long interaction history, reason over partially observable states, and recover from its own previous mistakes. 
An early suboptimal action can change the future observation stream, causing errors to compound across the trajectory~\citep{zhang2025agentracer}.

To improve multi-turn agent performance, prior work has mainly explored two training paradigms. 
The first is multi-turn SFT~\citep{teng2024fine, li2025beyond, zhou2025mixture}, which decomposes teacher trajectories into turn-level training examples and trains the model to imitate them. 
This can improve formatting and local action prediction, but it remains off-policy: the student is optimized on teacher-generated states rather than the histories induced by its own behavior. 
The second is reinforcement learning with verifiable rewards (RLVR)~\citep{shao2024deepseekmath, zhou2025disco, yang2025let}, where the agent receives reward from the final environment outcome~\citep{deepswe2025, wei2025webagent}. 
While this directly optimizes task success, it typically requires an external verifier or environment-defined reward signal, and the resulting reward is often sparse relative to the long sequence of student-generated turns, making credit assignment difficult~\citep{wang2026openclaw}. 
These limitations motivate on-policy distillation methods that can provide teacher supervision on student-induced trajectories while avoiding the sparsity of pure outcome-based reward learning.

\paragraph{On-Policy Distillation and Multi-turn Limitations.}
To mitigate the exposure bias of offline imitation, OPD shifts supervision from teacher-generated trajectories to the student's own generated distribution~\citep{agarwal2024policy, lu2025onpolicydistillation, zhou2026omniopd}. 
In OPD, the student first samples trajectories under its current policy, and the teacher then provides token-level feedback on the resulting student-induced states. 
This aligns training with inference-time behavior and allows the student to learn from states that arise from its own decisions, rather than only from expert demonstrations~\citep{zhao2026self, he2026self, hubotter2026reinforcement, hao2026self}.

Despite this advantage, recent analyses have revealed that the token-level logit signal underlying standard OPD is surprisingly brittle.
This brittleness is further amplified in multi-turn settings, where local token-level mismatches can compound across turns. 
First, OPD relies on a narrow token-level overlap between the teacher's and student's plausible next-token distributions~\citep{li2026rethinking}. 
When multiple continuations are semantically valid, dense token-level alignment can over-penalize acceptable alternatives. 
Second, OPD can reinforce degenerate local patterns when repeated or low-quality prefixes accidentally receive favorable token-level gradients~\citep{luo2026demystifying}; in multi-turn agents, this may manifest as repeated thoughts or repeated actions across consecutive turns. 
And third, multi-turn interaction amplifies distribution shift through compounding errors, where student thoughts or actions can push the environment history outside the teacher's effective state distribution, making later teacher supervision less reliable~\citep{wang2026tcod, li2025beyond}.

A concurrent line of our work, TCOD~\citep{wang2026tcod}, addresses this instability by controlling the trajectory depth and gradually expanding the student-controlled horizon through a temporal curriculum. 
In contrast, \ours does not rely on a manually designed curriculum over turn depth. 
Instead, we delegate the intervention decision to the teacher, i.e., for each turn, the teacher determines whether supervision should be skipped, weakly applied, or strongly applied. 
To further account for compounding errors that may move the student into states where the teacher is less reliable, \ours weights the token-level OPD signal by the teacher's predictive confidence. 
Thus, rather than deciding \emph{how long} to apply OPD, \ours decides \emph{where} OPD should be applied and how strongly it should be trusted.
\section{Methodology}
\label{sec:workflow}
\subsection{Preliminaries}

\paragraph{Multi-turn Agent Interaction.}
We consider an agent that interacts with an external environment over multiple turns.
At turn $t$, the agent receives an observation $o_t$ from the environment and conditions on the full interaction history
\begin{equation}
    h_t = (o_0, a_0, o_1, a_1, \ldots, o_{t-1}, a_{t-1}, o_t),
\end{equation}
where $a_i$ denotes the model response at turn $i$. 
The response $a_t$ may include natural language reasoning, tool-use arguments, and an executable action. 
After the executable part of $a_t$ is submitted to the environment, the environment returns the next observation:
\begin{equation}
    o_{t+1} \sim \mathcal{E}(o_{t+1} \mid h_t, a_t).
\end{equation}

A complete trajectory containing $T$ turns is therefore
\begin{equation}
    \tau = (h_0, a_0, h_1, a_1, \ldots, h_{T-1}, a_{T-1}),
\end{equation}
which terminates either when a termination action is taken or when the maximum horizon $T$ is reached. 
This dependency distinguishes multi-turn agent learning from single-turn generation: an early mistake can alter the environment state and affect all subsequent observations.

\paragraph{Multi-turn On-Policy Distillation.}
Let $\pi_\theta$ denote the student policy and $\pi_\phi$ denote the teacher policy. 
In multi-turn OPD, trajectories are sampled from the student policy rather than from the teacher, i.e.,
\begin{equation}
    a_t \sim \pi_\theta(\cdot \mid h_t), \quad t = 0,\ldots,T-1.
\end{equation}
The teacher then provides supervision on the student-induced states. A standard multi-turn OPD objective applies KL-style policy distillation across the student trajectory:
\begin{equation}
    \mathcal{L}_{\mathrm{OPD}}(\theta)
    =
    \mathbb{E}_{\tau \sim \pi_\theta}
    \left[
    \sum_{t=0}^{T-1}
    D_{\mathrm{KL}}
    \left(
    \pi_\theta(\cdot \mid h_t)
    \;\middle\|\;
    \pi_\phi(\cdot \mid h_t)
    \right)
    \right],
    \label{eq:opd_loss}
\end{equation}
The student is trained on its own state distribution, allowing the teacher to provide feedback on states that arise from the student's behavior instead of the teacher-generated trajectories.

However, applying OPD uniformly to every turn in a multi-turn trajectory can be suboptimal.
Early student errors may alter future observations and push later histories away from the teacher's reliable state distribution, making subsequent supervision less dependable.
Meanwhile, token-level OPD can be brittle when teacher and student policies assign probability mass to different but semantically valid continuations, or when dense alignment reinforces local degeneracies such as repeated thoughts or actions.
These issues motivate a selective mechanism that decides when teacher distillation should be applied and how strongly the resulting signal should be trusted.

\subsection{\ours}
\label{sec:method}
\begin{figure*}[t]
    \centering
    \includegraphics[width=\textwidth]{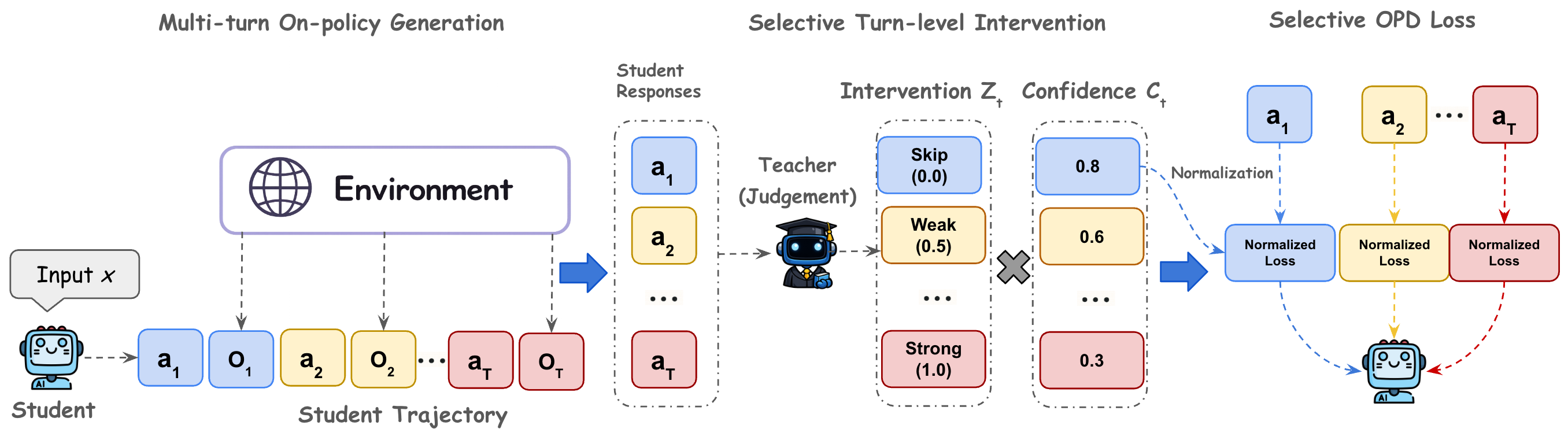}
    \caption{
    Overview of \ours.
    The student first performs multi-turn on-policy generation by interacting with the environment, producing a trajectory segmented into alternating actions and observations.
    For each student turn, \ours{} uses environment feedback and teacher judgment to determine whether the response should be skipped or intervened on, and combines this intervention decision with the teacher confidence score.
    The resulting selective weight modulates the OPD loss for each turn, so teacher supervision is applied more strongly when intervention is needed and the teacher signal is reliable.
    }
    \label{fig:our_framework}
\end{figure*}
The limitations of dense multi-turn OPD suggest that teacher supervision should be applied selectively rather than uniformly across all turns.
Standard OPD treats a student trajectory as a flat sequence of tokens and applies teacher distillation everywhere, which makes training sensitive to token-level mismatch even when the student response is semantically acceptable.
Multi-turn interaction provides a natural way to move beyond this flat token-level view: the trajectory is already segmented by environment feedback, where each turn produces a response $a_t$, receives an observation, and transitions to a new history $h_{t+1}$.
This turn structure allows the teacher to critique larger units of behavior and decide whether token-level distillation is needed for the current student response.
At the same time, teacher supervision should not be trusted uniformly across the trajectory.
Because student errors can compound over multiple turns, later histories may drift away from the teacher's reliable state distribution, making token-level supervision less dependable.

Motivated by these observations, and illustrated in \figref{fig:our_framework}, we propose \ours, a selective intervention framework for multi-turn OPD.
For each turn, \ours{} observes the environment feedback and uses teacher judgment to decide whether the current student response should be skipped or intervened on.
It then weighs the token-level OPD objective by the teacher's predictive confidence, so the supervision is the strongest when intervention is needed and the teacher signal is reliable.
Finally, \ours{} applies loss normalization to preserve the overall loss scale of standard OPD while retaining selective turn-level weighting.
The full training procedure is presented in Algorithm~\ref{alg:sl_opd} in \appref{sec:implementation_details}.

\subsection{Turn-Level Intervention}
\label{sec:method_intervention}
To obtain turn-level supervision without a task verifier, we introduce a discrete intervention variable
\begin{equation}
    z_t \in \{\textsc{Skip}, \textsc{Weak}, \textsc{Strong}\},
    \label{eq:intervention_labels}
\end{equation}
which determines how strongly teacher distillation should be applied at turn $t$. 
Each intervention label is mapped to a scalar intervention weight:
\begin{equation}
    i_t (h_t, a_t) =
    \begin{cases}
        0, & z_t = \textsc{Skip}, \\
        \alpha, & z_t = \textsc{Weak}, \\
        1, & z_t = \textsc{Strong},
    \end{cases}
    \label{eq:mapping}
\end{equation}
where $\alpha \in (0,1)$ controls the strength of weak intervention. 
The \textsc{Skip} label suppresses OPD at the current turn, \textsc{Weak} applies partial teacher supervision, and \textsc{Strong} applies full teacher supervision.

We assign $z_t$ through a two-stage verifier-free procedure.
First, we inspect the environment feedback for deterministic interaction failures. 
If the student response produces a clearly invalid transition, such as a failed tool call, missing executable action, or other environment-detectable execution failure, we assign
$z_t = \textsc{Strong}$, i.e., $i_t = 1$.
These cases do not require 
further second-stage judgment,
since the environment already indicates that the current response should be corrected.

For all remaining turns, where the response is executable but its quality is ambiguous, we query the teacher with a proactive intervention-evaluation prompt (see details in \appref{sec:appendix_prompts}). 
Given the interaction history $h_t$, the student response $a_t$, and the available environment feedback, the teacher classifies the turn into exactly one of $\{\textsc{Skip}, \textsc{Weak}, \textsc{Strong}\}$. 
The resulting label is then converted into the intervention weight $i_t$ using the mapping in \eqnref{eq:mapping}. 
Intuitively, \textsc{Skip} is used when the response is acceptable or when multiple strategies are plausible; \textsc{Weak} is used when the response may be suboptimal but the teacher is not fully certain; and \textsc{Strong} is used when the response is clearly detrimental and teacher correction is necessary. 
%

\subsection{Teacher Confidence Weighting}
The intervention label determines whether teacher supervision is needed at a given turn. 
However, even when the teacher identifies a turn as requiring intervention, its token-level supervision may not always be equally reliable. 
Due to compounding errors in multi-turn interaction, the student's previous actions may induce histories that deviate from the teacher's reliable state distribution.
In such cases, the teacher may still recommend intervention, but its predictive distribution can be uncertain. 
We therefore introduce a second factor, denoted as $c_t$, to measure the teacher's confidence in the supervision it provides at turn $t$.

Specifically, for each token position $i$ in the student response $a_t$, let
\begin{equation}
    p^{\star}_{t,i}
    =
    \max_{\mathcal{V}}
    \pi_{\phi}(v \mid h_t, a_{t,<i})
\end{equation}
denote the teacher's top-1 probability over the vocabulary $\mathcal{V}$.
This value measures how decisive the teacher is about its most likely next token at position $i$. 
We define the turn-level teacher confidence score as the average top-1 probability across the student response for this turn, i.e.,
\begin{equation}
    c_t(h_t,a_t)
    =
    \frac{1}{|a_t|}
    \sum_{i=1}^{|a_t|}
    p^{\star}_{t,i}
    \in (0,1].
\end{equation}
A sharp teacher distribution yields a high $c_t$, indicating that the teacher is confident about the token-level supervision for this turn. 
Conversely, a flat distribution yields a low $c_t$, indicating that the teacher itself is uncertain and that the resulting distillation signal should be weakened. 
This confidence term is especially important for later or off-distribution histories, where teacher intervention may be necessary but less reliable.

We combine the turn-level intervention weight $i_t$ with the teacher confidence score $c_t$ to obtain the final \textit{selective OPD weight}:
\begin{equation}
    s_t = i_t \cdot c_t
    \label{eq:selective_opd_weight}
\end{equation}
The same weight $s_t$ is applied uniformly to all token-level OPD loss within turn $t$. 
Note that as $p^{\star}_{t,i}$ is obtained from the same teacher distribution used for the standard distillation loss, computing $c_t$ introduces no additional teacher forward passes beyond the intervention query described in \secref{sec:method_intervention}.
%

\subsection{Loss Normalization}
The selective weight $s_t \in [0,1]$ controls how strongly teacher supervision is applied at each token of the student response in turn $t$. 
However, directly multiplying the OPD loss (\eqnref{eq:opd_loss}) by $s_t$ reduces the overall gradient magnitude relative to standard OPD. 
Without normalization, selective weighting changes both the relative contribution of different turns and the overall magnitude of the training objective.
This makes it difficult to attribute performance differences specifically to selective supervision rather than to a change in optimization scale.
To isolate the contribution of selective weighting, we normalize the loss scale within each batch.

Let $\ell_{\mathrm{OPD}}(h_t,a_t)$ denote the token-level OPD loss for turn $t$, and let $N$ denote the total number of tokens in the student response of the current turn. 
We define the normalized selective OPD objective as
\begin{equation}
    \mathcal{L}_{\ours}
    =
    Z
    \sum_{t=0}^{T-1}
    s_t \,
    \mathcal{L}_{\mathrm{OPD}}(h_t,a_t),
\end{equation}
where the normalization factor is
\begin{equation}
    Z
    =
    \frac{N}
    {
    \max\left(
    \sum_{t=0}^{T-1} s_t |a_t|,
    \epsilon
    \right)
    }.
\end{equation}
Here, $|a_t|$ is the number of student tokens in turn $t$, and $\epsilon$ is a small constant for numerical stability. 
This normalization keeps the average effective token weight close to one, matching the scale of standard OPD while preserving the relative turn-level weighting induced by $s_t$.
As a result, \ours provides a verifier-free turn-level supervision policy that selectively reallocates teacher supervision without merely reducing the overall training signal.

\section{Theoretical Analysis}
\label{sec:theory}
In this section, we provide a theoretical analysis to analyze the role of selective weighting in \ours.
Our goal is not to prove that the teacher intervention gate is always correct, but to show how the proposed weights affect noisy teacher supervision and how the selective objective relates to an ideal oracle intervention objective.
Throughout this section, let $\ell_t(\theta)$ denote the per-turn OPD loss at turn $t$, and let the final selective weight be
\begin{equation}
    s_t = i_t \cdot c_t \in [0,1],
\end{equation}
where $i_t$ is the intervention weight and $c_t$ is the teacher confidence score, as in \eqnref{eq:selective_opd_weight}.
And the normalized selective OPD objective is
\begin{equation}
    \mathcal{L}_{\ours}
    =
    Z
    \sum_{t=0}^{T-1}
    s_t \ell_t(\theta),
    \quad
    \mathrm{where~}
    Z =
    \frac{N}{
    \max\left(
    \sum_{t=0}^{T-1} s_t |a_t|,
    \epsilon
    \right)
    }.
\end{equation}

\paragraph{Selective Reweighting of Supervision Noise.}
We first analyze how selective weighting changes the relative influence of noisy teacher signals.
Suppose the observed per-turn OPD loss decomposes as
\begin{equation}
    \ell_t(\theta)
    =
    \ell_t^\star(\theta) + \xi_t,
\end{equation}
where $\ell_t^\star$ is the ideal supervision loss and $\xi_t$ is zero-mean noise caused by teacher uncertainty, token-level mismatch, or off-distribution histories.
Assume
$
    \mathbb{E}[\xi_t]=0,
    \mathrm{Var}[\xi_t]\leq \sigma_t^2$.

\begin{proposition}[Relative Suppression of Noisy Supervision]
For a fixed batch, define the normalized effective turn weight as
$
    \widetilde{s}_t = Zs_t$.
Then the contribution of turn $t$ to the variance of the noisy component in the normalized selective objective is bounded by
$
    \widetilde{s}_t^2 \sigma_t^2$.
Moreover, for any two turns $p$ and $q$, their relative variance contributions satisfy
\begin{equation}
    \frac{\widetilde{s}_p^2 \sigma_p^2}
         {\widetilde{s}_q^2 \sigma_q^2}
    =
    \frac{s_p^2 \sigma_p^2}
         {s_q^2 \sigma_q^2}.
\end{equation}
\end{proposition}

\begin{proof}
The noisy component of the normalized selective objective is
\begin{equation}
    Z\sum_{t=0}^{T-1}s_t\xi_t
    =
    \sum_{t=0}^{T-1}\widetilde{s}_t\xi_t.
\end{equation}
The variance contribution of turn $t$ is therefore
\begin{equation}
    \mathrm{Var}[\widetilde{s}_t\xi_t]
    =
    \widetilde{s}_t^2\mathrm{Var}[\xi_t]
    \leq
    \widetilde{s}_t^2\sigma_t^2.
\end{equation}
For two turns $p$ and $q$, the common normalization factor $Z$ cancels:
\begin{equation}
    \frac{\widetilde{s}_p^2 \sigma_p^2}
         {\widetilde{s}_q^2 \sigma_q^2}
    =
    \frac{Z^2s_p^2 \sigma_p^2}
         {Z^2s_q^2 \sigma_q^2}
    =
    \frac{s_p^2 \sigma_p^2}
         {s_q^2 \sigma_q^2}.
\end{equation}
\end{proof}

\paragraph{Remark.}
This proposition does not claim that the total variance of \ours is always smaller than that of standard OPD.
Because vanilla-scale normalization may produce $Z>1$, the absolute variance of the normalized objective can increase in some batches.
Instead, the result shows that normalization preserves the relative suppression induced by $s_t$: turns assigned smaller selective weights contribute quadratically less noise than turns assigned larger weights.
Thus, when unreliable or unnecessary turns receive smaller $s_t$, their noisy teacher supervision is reduced relative to more reliable intervention turns.

\paragraph{Connection to Oracle Intervention.}
We next relate the selective objective to an ideal oracle intervention objective.
Let
$
    m_t^\star \in \{0,1\}$
denote an unobserved oracle intervention indicator, where $m_t^\star=1$ means that teacher supervision is useful at turn $t$ and $m_t^\star=0$ means that the turn should be skipped.
The corresponding oracle objective is
\begin{equation}
    \mathcal{L}_{\mathrm{oracle}}(\theta)
    =
    \sum_{t=0}^{T-1}
    m_t^\star \ell_t(\theta).
\end{equation}
Our method can be viewed as replacing this binary oracle mask with the verifier-free soft weight $s_t=i_tc_t$.
The following proposition gives a simple error decomposition for this approximation.

\begin{proposition}[Approximation Error of Selective Intervention]
Assume the per-turn OPD loss is bounded:
\begin{equation}
    0 \leq \ell_t(\theta) \leq B.
\end{equation}
Then, for any trajectory, the approximation error of the selective objective relative to the oracle objective is bounded as
\begin{equation}
    \left|
    \sum_{t=0}^{T-1}
    s_t \ell_t(\theta)
    -
    \sum_{t=0}^{T-1}
    m_t^\star \ell_t(\theta)
    \right|
    \leq
    B
    \sum_{t=0}^{T-1}
    |s_t - m_t^\star|.
\end{equation}
\end{proposition}

\begin{proof}
By the triangle inequality,
\begin{align}
    \left|
    \sum_{t=0}^{T-1}
    s_t \ell_t(\theta)
    -
    \sum_{t=0}^{T-1}
    m_t^\star \ell_t(\theta)
    \right|
    &=
    \left|
    \sum_{t=0}^{T-1}
    (s_t-m_t^\star)\ell_t(\theta)
    \right| \\
    &\leq
    \sum_{t=0}^{T-1}
    |s_t-m_t^\star|\,|\ell_t(\theta)|
    \leq
    B
    \sum_{t=0}^{T-1}
    |s_t-m_t^\star|.
\end{align}
\end{proof}

\paragraph{Remark.}
This proposition should be interpreted as an error decomposition rather than a guarantee that \ours recovers the oracle mask.
It shows that the deviation from oracle selective supervision is controlled by how closely the soft weights $s_t$ approximate the ideal intervention decisions $m_t^\star$.
In \ours, deterministic environment failures provide high-precision positive intervention signals, while teacher judgments approximate the oracle decision for executable but ambiguous turns.
Teacher confidence further softens the intervention weight when the token-level teacher distribution is uncertain.
Therefore, the analysis highlights the intended role of the intervention mechanism: it should identify turns where teacher supervision is both necessary and reliable, rather than applying dense OPD uniformly across the entire trajectory.
\section{Experiments}
\label{sec:experiments}

\subsection{Experimental Setup}
\label{sec:setup}

\paragraph{Models.}
Following prior work on multi-turn OPD~\citep{wang2026tcod}, we use teacher and student models from the same model family to reduce confounding effects from tokenizer, architecture, and instruction-format mismatch.
Specifically, we use Qwen3-32B and Qwen3-8B as teacher models, and Qwen3-0.6B and Qwen3-1.7B as student models~\citep{yang2025qwen3}. 
This yields multiple teacher--student pairs with different capacity gaps, allowing us to evaluate whether \ours remains effective across both large-gap and moderate-gap distillation settings.

\paragraph{Benchmarks and Evaluation Metrics.}
We evaluate \ours on three multi-turn agent benchmarks: ALFWorld~\citep{shridhar2020alfworld}, ScienceWorld~\citep{wang2022scienceworld}, and SearchQA~\citep{jin2025search}. 
These benchmarks cover embodied interaction, scientific task solving, and search-based question answering, respectively, allowing us to assess the effectiveness of selective intervention across different forms of multi-turn decision making.
For ALFWorld, we report success rate (SR) on both the seen and unseen splits. 
For ScienceWorld, we report success rate (SR) and the mean score on the test set. 
And for SearchQA, we report exact match (EM) across different question-answering datasets. 
Detailed definitions of the evaluation metrics are provided in \appref{sec:implementation_details}.

\paragraph{Baselines.}
We compare \ours with four main baselines: off-policy SFT, standard OPD~\citep{lu2025onpolicydistillation}, TCOD-F2B~\citep{wang2026tcod}, and Entropy-Aware OPD~\citep{jin2026entropy}. 
TCOD-F2B stabilizes multi-turn OPD through a forward-to-backward temporal curriculum, and Entropy-Aware OPD augments OPD with entropy-gated KL loss originally designed for single-turn distillation. 
Since \ours uses one additional teacher query to obtain the turn-level intervention label, we include an additional compute-matched baseline, OPD with two rollouts, where both the student and teacher are queried twice during OPD data collection without selective intervention. 
This controls for the additional teacher-query budget.

\paragraph{Training Details.}
We implement OPD with a fully asynchronous training pipeline. 
Unless otherwise specified, all methods are trained for one epoch with learning rate $1\times10^{-6}$ and rollout number 1. 
For \ours, we set the weak-intervention coefficient to $\alpha=0.5$. 
We analyze the sensitivity to this choice in \secref{sec:sensitivity_analysis}. 
Additional implementation details are provided in \appref{sec:implementation_details}.

\subsection{Main Results}
\label{sec:result_analysis}
\begin{table*}[t]
\centering
\caption{
    Main results on ALFWorld. 
    We report success rate (SR) and average number of turns on both seen and unseen splits. 
    Best and second-best SR among distillation methods for each setup is in \textbf{bold} and \underline{underlined}, respectively.
}
\vspace{-0.1in}
\label{tab:main_results_alfworld}
\setlength{\tabcolsep}{14pt}
\resizebox{\textwidth}{!}{%
\begin{tabular}{lllcccc}
\toprule
\textbf{Student Model} & \textbf{Teacher Model} & \textbf{Method}
& \textbf{Seen SR $\uparrow$} & \textbf{Seen Turns $\downarrow$}
& \textbf{Unseen SR $\uparrow$} & \textbf{Unseen Turns $\downarrow$} \\
\midrule
\multicolumn{7}{c}{\textit{Reference Models (Base Inference)}} \\
\midrule
Qwen3-0.6B & -- & Base & 1.43 & 29.65 & 0.75 & 29.82 \\
Qwen3-1.7B & -- & Base & 25.71 & 24.63 & 26.12 & 24.72 \\
-- & Qwen3-8B & Base & 61.43 & 17.95 & 67.16 & 17.24 \\
-- & Qwen3-32B & Base & 57.14 & 18.94 & 69.40 & 18.21 \\
\midrule
\multicolumn{7}{c}{\textit{Student: Qwen3-0.6B, Teacher: Qwen3-8B}} \\
\midrule
\multirow{6}{*}{Qwen3-0.6B} & \multirow{6}{*}{Qwen3-8B} & SFT & 2.86 & 29.32 & 4.48 & 28.87 \\
 & & OPD & 55.00 & 18.53 & 56.72 & 20.30 \\
 & & OPD (rollout 2) & \underline{56.43} & 18.76 & 59.70 & 20.31 \\
 & & Entropy-Aware OPD & 54.26 & 18.12 & 54.48 & 20.75 \\
 & & TCOD-F2B & \underline{56.43} & 18.51 & \underline{61.19} & 18.82 \\
 & & \ours{} & \textbf{57.86} & 18.46 & \textbf{61.94} & 18.89 \\
\midrule
\multicolumn{7}{c}{\textit{Student: Qwen3-1.7B, Teacher: Qwen3-8B}} \\
\midrule
\multirow{6}{*}{Qwen3-1.7B} & \multirow{6}{*}{Qwen3-8B} & SFT & 25.00 & 24.58 & 26.12 & 24.54 \\
 & & OPD & 60.00 & 17.80 & 61.94 & 17.67 \\
 & & OPD (rollout 2) & 63.00 & 17.51 & 64.93 & 17.87 \\
 & & Entropy-Aware OPD & 63.57 & 17.28 & \underline{65.67} & 18.20 \\
 & & TCOD-F2B & \textbf{64.29} & 17.65 & 58.21 & 19.16 \\
 & & \ours{} & \textbf{64.29} & 18.14 & \textbf{70.15} & 17.86 \\
\midrule
\multicolumn{7}{c}{\textit{Student: Qwen3-0.6B, Teacher: Qwen3-32B}} \\
\midrule
\multirow{5}{*}{Qwen3-0.6B} & \multirow{5}{*}{Qwen3-32B} & OPD & 53.57 & 19.49 & 68.66 & 18.65 \\
 & & OPD (rollout 2) & 52.14 & 20.28 & \underline{70.37} & 18.21 \\
 & & Entropy-Aware OPD & 48.57 & 20.87 & 68.13 & 18.28 \\
 & & TCOD-F2B & \textbf{54.29} & 20.18 & 67.91 & 19.26 \\
 & & \ours{} & \textbf{54.29} & 19.85 & \textbf{70.90} & 17.98 \\
\midrule
\multicolumn{7}{c}{\textit{Student: Qwen3-1.7B, Teacher: Qwen3-32B}} \\
\midrule
\multirow{5}{*}{Qwen3-1.7B} & \multirow{5}{*}{Qwen3-32B} & OPD & 58.57 & 18.89 & \underline{74.85} & 17.16 \\
 & & OPD (rollout 2) & \underline{62.14} & 18.17 & 72.39 & 18.37 \\
 & & Entropy-Aware OPD & 60.00 & 18.80 & \textbf{79.10} & 16.75 \\
 & & TCOD-F2B & 57.14 & 19.42 & 73.13 & 17.87 \\
 & & \ours{} & \textbf{62.82} & 18.75 & 73.88 & 17.57 \\
\bottomrule
\end{tabular}%
}
\end{table*}
\begin{table*}[t]
\centering
\caption{
    Main results on ScienceWorld. 
    We report success rate (SR), mean task score, and average number of turns. 
    Best and second-best SR among distillation methods for each setup is in \textbf{bold} and \underline{underlined}, respectively.
}
\vspace{-0.1in}
\label{tab:main_results_scienceworld}
\setlength{\tabcolsep}{22pt}
\resizebox{\textwidth}{!}{%
\begin{tabular}{lllccc}
\toprule
\textbf{Student Model} & \textbf{Teacher Model} & \textbf{Method}
& \textbf{SR $\uparrow$} & \textbf{Mean Score $\uparrow$} & \textbf{Turns $\downarrow$} \\
\midrule
\multicolumn{6}{c}{\textit{Reference Models (Base Inference)}} \\
\midrule
Qwen3-0.6B & -- & Base & 3.36 & -43.51 & 21.38 \\
Qwen3-1.7B & -- & Base & 3.36 & -24.00 & 21.38 \\
-- & Qwen3-8B & Base & 10.74 & -16.09 & 21.67 \\
-- & Qwen3-32B & Base & 15.44 & -17.92 & 18.67 \\
\midrule
\multicolumn{6}{c}{\textit{Student: Qwen3-0.6B, Teacher: Qwen3-8B}} \\
\midrule
\multirow{5}{*}{Qwen3-0.6B} & \multirow{5}{*}{Qwen3-8B} & SFT & 2.01 & -31.84 & 23.94 \\
 & & OPD & 2.01 & -12.88 & 23.94 \\
 & & Entropy-Aware OPD & \textbf{3.36} & -11.52 & 23.44 \\
 & & TCOD-F2B & \underline{3.03} & -2.07 & 24.88 \\
 & & \ours{} & 2.01 & \textbf{+3.89} & 26.94 \\
\midrule
\multicolumn{6}{c}{\textit{Student: Qwen3-1.7B, Teacher: Qwen3-8B}} \\
\midrule
\multirow{5}{*}{Qwen3-1.7B} & \multirow{5}{*}{Qwen3-8B} & SFT & 2.01 & -17.56 & 23.28 \\
 & & OPD & 4.70 & -2.78 & 25.12 \\
 & & Entropy-Aware OPD & 3.36 & -3.13 & 25.91 \\
 & & TCOD-F2B & 4.70 & -9.68 & 24.06 \\
 & & \ours{} & \textbf{6.71} & \textbf{+0.24} & 25.30 \\
\midrule
\multicolumn{6}{c}{\textit{Student: Qwen3-0.6B, Teacher: Qwen3-32B}} \\
\midrule
\multirow{5}{*}{Qwen3-0.6B} & \multirow{5}{*}{Qwen3-32B} & SFT & 2.01 & -34.59 & 18.10 \\
 & & OPD & 3.36 & -28.46 & 20.79 \\
 & & Entropy-Aware OPD & 6.04 & -26.10 & 20.00 \\
 & & TCOD-F2B & 2.68 & -33.42 & 23.20 \\
 & & \ours{} & \textbf{6.04} & \textbf{-26.10} & 20.00 \\
\midrule
\multicolumn{6}{c}{\textit{Student: Qwen3-1.7B, Teacher: Qwen3-32B}} \\
\midrule
\multirow{5}{*}{Qwen3-1.7B} & \multirow{5}{*}{Qwen3-32B} & SFT & 2.01 & -15.61 & 23.23 \\
 & & OPD & 4.03 & -29.89 & 19.64 \\
 & & Entropy-Aware OPD & 3.36 & -14.99 & 23.59 \\
 & & TCOD-F2B & \underline{6.04} & \textbf{-14.89} & 22.02 \\
 & & \ours{} & \textbf{7.38} & -16.01 & 20.91 \\
\bottomrule
\end{tabular}%
}
\end{table*}
\begin{table*}[htbp]
\centering
\caption{
    Main results on SearchQA. 
    We report exact match (EM) across seven question-answering datasets and their average. 
    Best and second-best SR among distillation methods for each setup is in \textbf{bold} and \underline{underlined}, respectively.
}
\vspace{-0.1in}
\label{tab:main_results_searchqa}
\setlength{\tabcolsep}{4pt}
\resizebox{\textwidth}{!}{%
\begin{tabular}{lllcccccccc}
\toprule
\textbf{Student Model} & \textbf{Teacher Model} & \textbf{Method}
& \textbf{NQ $\uparrow$} & \textbf{TriviaQA $\uparrow$} & \textbf{PopQA $\uparrow$} & \textbf{HotpotQA $\uparrow$}
& \textbf{2WikiMQA $\uparrow$} & \textbf{Musique $\uparrow$} & \textbf{Bamboogle $\uparrow$} & \textbf{Avg. $\uparrow$} \\
\midrule
\multicolumn{11}{c}{\textit{Reference Models (Base Inference)}} \\
\midrule
Qwen3-0.6B & -- & Base & 5.60 & 15.40 & 0.40 & 6.00 & 8.40 & 1.20 & 3.20 & 5.74 \\
Qwen3-1.7B & -- & Base & 24.20 & 36.40 & 5.20 & 19.00 & 14.20 & 10.80 & 18.40 & 18.31 \\
-- & Qwen3-8B & Base & 34.60 & 53.00 & 14.00 & 38.20 & 37.40 & 28.80 & 45.60 & 35.94 \\
-- & Qwen3-14B & Base & 35.00 & 55.80 & 15.60 & 41.20 & 43.00 & 30.00 & 49.60 & 38.60 \\
\midrule
\multicolumn{11}{c}{\textit{Student: Qwen3-0.6B, Teacher: Qwen3-8B}} \\
\midrule
\multirow{5}{*}{Qwen3-0.6B} & \multirow{5}{*}{Qwen3-8B} & SFT & 14.40 & 23.60 & 5.80 & 13.20 & 11.69 & 2.23 & 6.40 & 11.05 \\
 & & OPD & 25.00 & 41.40 & 17.80 & 28.20 & 24.80 & 17.00 & 24.00 & 25.46 \\
 & & Entropy-Aware OPD & \textbf{25.60} & 40.00 & \textbf{18.80} & 28.80 & \textbf{25.60} & \underline{17.40} & \textbf{24.40} & \textbf{25.80} \\
 & & TCOD-F2B & \underline{25.40} & 41.00 & 15.60 & \underline{29.00} & 24.60 & 16.80 & 23.20 & 25.09 \\
 & & \ours{} & 25.20 & \textbf{41.60} & \underline{18.40} & \textbf{29.80} & 23.20 & 15.00 & \underline{24.00} & \underline{25.31} \\
\midrule
\multicolumn{11}{c}{\textit{Student: Qwen3-1.7B, Teacher: Qwen3-8B}} \\
\midrule
\multirow{5}{*}{Qwen3-1.7B} & \multirow{5}{*}{Qwen3-8B} & SFT & 27.40 & 39.60 & 11.60 & 21.20 & 8.80 & 12.00 & 20.00 & 20.09 \\
 & & OPD & 30.60 & 47.20 & 12.40 & 33.20 & 29.60 & \textbf{24.80} & 32.80 & 30.09 \\
 & & Entropy-Aware OPD & 29.40 & \underline{48.20} & 12.60 & 33.20 & \underline{30.20} & 23.80 & \textbf{36.80} & \underline{30.60} \\
 & & TCOD-F2B & \textbf{31.20} & 47.60 & 12.40 & \textbf{33.80} & 28.60 & 24.00 & \underline{36.00} & 30.51 \\
 & & \ours{} & \underline{30.00} & \textbf{48.40} & \textbf{12.80} & 33.20 & \textbf{30.40} & \underline{24.60} & \underline{36.00} & \textbf{30.77} \\
\bottomrule
\end{tabular}%
}
\end{table*}
We present the main results on ALFWorld, ScienceWorld, and SearchQA in \tabref{tab:main_results_alfworld}, \ref{tab:main_results_scienceworld}, and \ref{tab:main_results_searchqa}, respectively.
Across all three benchmarks, \ours substantially improves over base student models and off-policy SFT. 
On ALFWorld, Qwen3-0.6B improves massively from $1.43/0.75$ (seen/unseen splits) SR to $57.86/61.94$ with a Qwen3-8B teacher, while reducing the average number of turns from nearly 30 to under 20. 
Compared with standard OPD, \ours further improves Qwen3-0.6B from $55.00$ to $57.86$ seen SR and from $56.72$ to $61.94$ unseen SR, while improving Qwen3-1.7B unseen SR from $61.94$ to $70.15$. 
Notably, these gains are not merely due to extra teacher queries, as \ours also outperforms the compute-matched OPD rollout-$2$ baseline in most ALFWorld settings, with comparable or fewer turns on the unseen split.
%

Compared with TCOD-F2B and Entropy-Aware OPD, \ours is strongest on ALFWorld and remains competitive on ScienceWorld and SearchQA. 
On ALFWorld, \ours achieves the best unseen SR in three out of four student--teacher settings and ties for the best seen SR in two settings, while generally maintaining similar turn efficiency to the strongest baselines. 
On SearchQA, \ours obtains the best average EM for Qwen3-1.7B distilled from Qwen3-8B, improving over OPD from $30.09$ to $30.77$. 
On ScienceWorld, \ours achieves the best SR for Qwen3-1.7B with both Qwen3-8B and Qwen3-32B teachers, and improves the Qwen3-0.6B mean score with Qwen3-8B from $-12.88$ under OPD to $+3.89$, although this comes with a higher average number of turns, suggesting more sustained task exploration.

\paragraph{Effects of Student and Teacher Scale.}
We notice that the benefits of \ours are more consistent for the larger Qwen3-1.7B student, which achieves strong gains on ALFWorld, the best average SearchQA result, and the best ScienceWorld SR across both teacher settings. 
For Qwen3-0.6B, OPD already provides large gains over SFT and base inference, while \ours further improves ALFWorld generalization but shows more task-dependent effects on ScienceWorld and SearchQA. 
We also observe that merely increasing teacher size is not uniformly beneficial, we hypothesize that stronger teacher signals may not be always more reliable when applied densely and that adaptive weighting is important for multi-turn OPD.

\begin{figure*}[t]
    \centering
    \includegraphics[width=\textwidth]{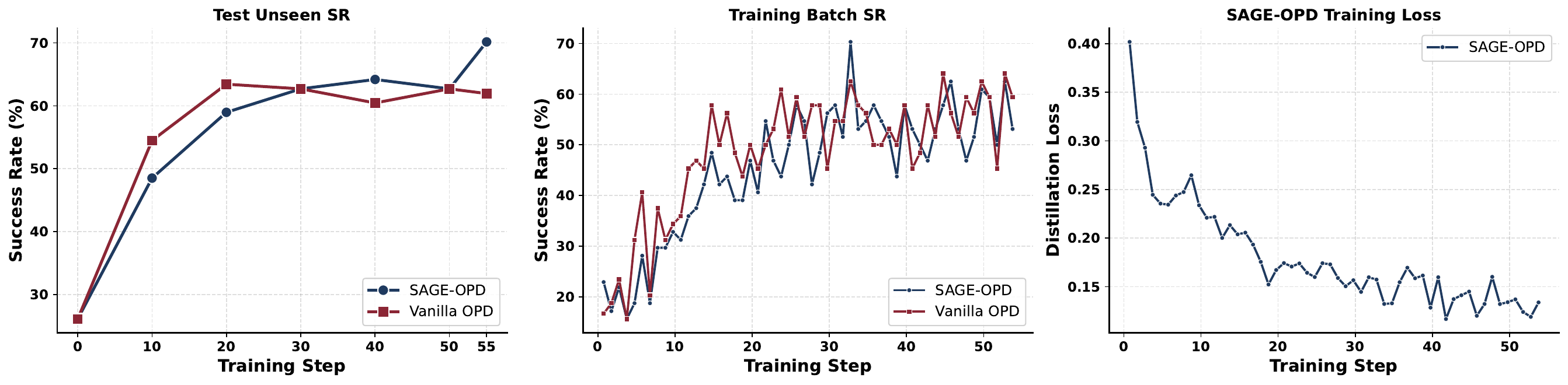}
    \caption{Training dynamics on ALFWorld comparing \ours{} with vanilla OPD. Left: test unseen success rate during training. Middle: training-batch success rate. Right: training loss of \ours{}.}
    \label{fig:training_dynamics}
\end{figure*}

\subsection{Training Dynamics}

\figref{fig:training_dynamics} compares the training dynamics of \ours and standard OPD~\citep{lu2025onpolicydistillation} on ALFWorld~\citep{shridhar2020alfworld}. 
On the unseen test split (\figref{fig:training_dynamics}-\textit{left}), both methods improve rapidly in the early stage, but standard OPD saturates after roughly 20--30 training steps and fluctuates around $60$--$63\%$ success rate. 
In contrast, \ours continues to improve in the later stage and reaches $70.15\%$ unseen success rate by the end of training. 
This suggests that selective intervention does not merely accelerate early learning, but also helps maintain useful teacher supervision after standard dense OPD begins to plateau.
%

The training-batch success rate (\figref{fig:training_dynamics}-\textit{middle}) shows the two methods tracking each other closely throughout training, so selective supervision fits the on-policy training distribution about as well as dense OPD while intervening on only a subset of turns.
This contrast (comparable training-batch success but a widening gap on the unseen split) shows that selective intervention improves generalization rather than training-distribution fit.
\figref{fig:training_dynamics}-\textit{right} further shows that the \ours distillation loss decreases steadily over training, indicating stable optimization under the intervention-weighted objective. 
Together, these dynamics support our hypothesis that selectively reallocating teacher supervision can improve generalization while preserving stable on-policy training.
\section{Ablation Studies and Analysis}
\label{sec:ablation_studies}

\begin{table*}[t]
\centering
\caption{
    Ablation study on ALFWorld. 
    We incrementally evaluate the contribution of turn-level intervention, teacher confidence weighting, and vanilla-scale normalization. 
    Best and second-best SR among distillation methods for each setup is in \textbf{bold} and \underline{underlined}, respectively.
}
\vspace{-0.1in}
\label{tab:ablation_alfworld}
\setlength{\tabcolsep}{12pt}
\resizebox{\textwidth}{!}{%
\begin{tabular}{lllcccc}
\toprule
\textbf{Student Model} & \textbf{Teacher Model} & \textbf{Method}
& \textbf{Seen SR $\uparrow$} & \textbf{Seen Turns $\downarrow$}
& \textbf{Unseen SR $\uparrow$} & \textbf{Unseen Turns $\downarrow$} \\
\midrule
\multicolumn{7}{c}{\textit{Student: Qwen3-0.6B, Teacher: Qwen3-8B}} \\
\midrule
\multirow{4}{*}{Qwen3-0.6B} & \multirow{4}{*}{Qwen3-8B}
& OPD & 55.00 & 18.53 & 56.72 & 20.30 \\
& & Intervention-only & \underline{57.86} & 17.80 & \textbf{62.69} & 18.49 \\
& & NoNorm & 55.71 & 19.19 & 55.22 & 20.50 \\
& & \ours{} & \textbf{57.86} & 18.46 & \underline{61.94} & 18.89 \\
\midrule
\multicolumn{7}{c}{\textit{Student: Qwen3-1.7B, Teacher: Qwen3-8B}} \\
\midrule
\multirow{4}{*}{Qwen3-1.7B} & \multirow{4}{*}{Qwen3-8B}
& OPD & 60.00 & 17.80 & 61.94 & 17.67 \\
& & Intervention-only & 59.29 & 18.39 & \underline{66.42} & 17.68 \\
& & NoNorm & \textbf{66.43} & 17.15 & 64.93 & 18.51 \\
& & \ours{} & \underline{64.29} & 18.14 & \textbf{70.15} & 17.86 \\
\midrule
\multicolumn{7}{c}{\textit{Student: Qwen3-0.6B, Teacher: Qwen3-32B}} \\
\midrule
\multirow{4}{*}{Qwen3-0.6B} & \multirow{4}{*}{Qwen3-32B}
& OPD & 53.57 & 19.49 & 68.66 & 18.65 \\
& & Intervention-only & 53.86 & 19.21 & 67.16 & 19.40 \\
& & NoNorm & \textbf{57.86} & 19.61 & \underline{69.40} & 18.10 \\
& & \ours{} & \underline{54.29} & 19.85 & \textbf{70.90} & 17.98 \\
\midrule
\multicolumn{7}{c}{\textit{Student: Qwen3-1.7B, Teacher: Qwen3-32B}} \\
\midrule
\multirow{4}{*}{Qwen3-1.7B} & \multirow{4}{*}{Qwen3-32B}
& OPD & 58.57 & 18.89 & \underline{74.85} & 17.16 \\
& & Intervention-only & \underline{62.14} & 18.29 & \textbf{79.10} & 16.12 \\
& & NoNorm & 61.43 & 18.39 & 73.13 & 17.14 \\
& & \ours{} & \textbf{62.82} & 18.75 & 73.88 & 17.57 \\
\bottomrule
\end{tabular}%
}
\end{table*}

\begin{figure*}[t]
    \centering
    \includegraphics[width=0.9\textwidth]{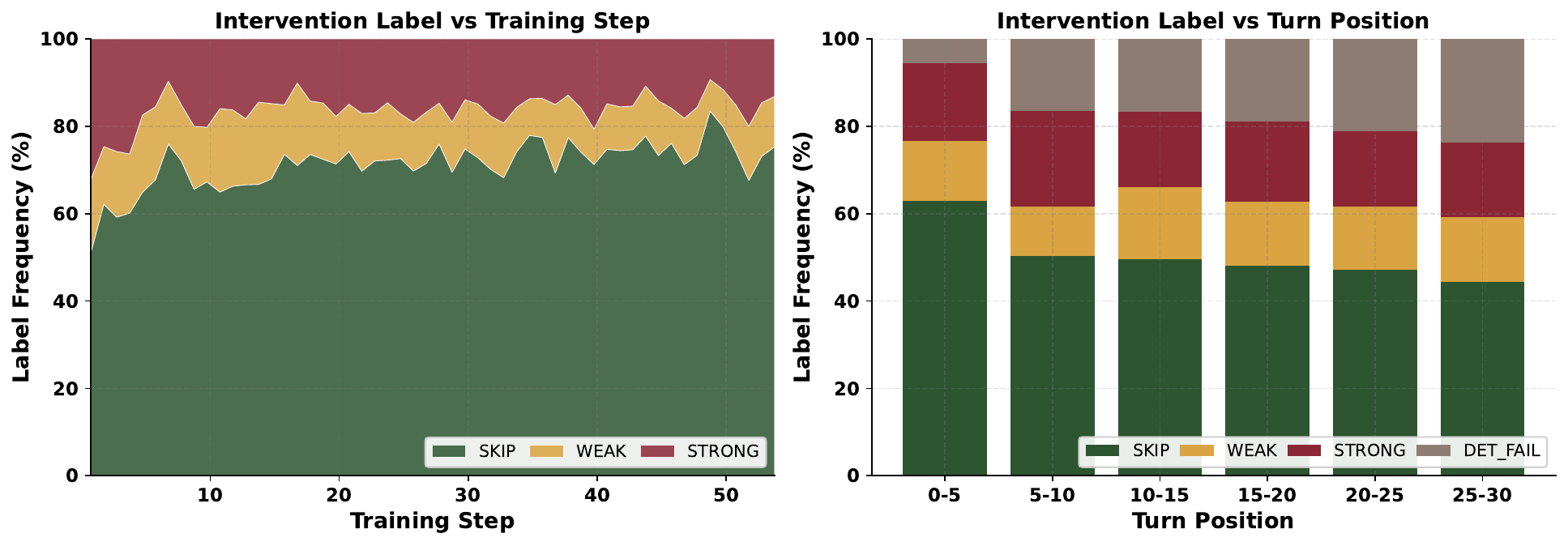}
    \caption{Intervention label distribution of \ours{} on ALFWorld. Left: label frequency across training steps. Right: label frequency across turn positions. \textsc{Det-Fail} denotes deterministic failures detected before teacher intervention, such as malformed actions, failed tool calls, missing executable actions, or environment-reported execution failures.}
    \label{fig:intervention_analysis}
\end{figure*}

\subsection{Intervention Analysis}

We first analyze the intervention labels (as in \eqnref{eq:intervention_labels}) produced by \ours across training steps and turn positions. 
In addition to the teacher-assigned labels \textsc{Skip}, \textsc{Weak}, and \textsc{Strong}, we also report \textsc{Det-Fail}, which corresponds to the deterministic-failure short circuit described in \secref{sec:method_intervention}. 
These are turns where the environment or action parser provides clear negative feedback, such as malformed actions, failed tool calls, missing executable actions, or execution failures. 
Such turns are directly assigned full intervention without requiring additional teacher judgment.

As shown in \figref{fig:intervention_analysis}-\textit{left}, as training progresses, the proportion of \textsc{Skip} labels generally increases, while \textsc{Weak} and \textsc{Strong} labels decrease. 
This suggests that the student produces more acceptable actions over time, and the intervention mechanism correspondingly reduces unnecessary teacher supervision. 
The turn-position analysis (\figref{fig:intervention_analysis}-\textit{right}) shows that \textsc{Skip} remains the dominant label across the trajectory, but gradually decreases as the turn index increases. 
Meanwhile, \textsc{Det-Fail} becomes more frequent in later turns, indicating that response quality becomes harder to maintain as the interaction horizon grows. 
This aligns with the compounding-error challenge in multi-turn agents. \ours will selectively skip acceptable turns while preserving full intervention for clear execution failures, which become more frequent as interaction histories grow longer.

\subsection{Ablation Studies}

We conduct ablation studies on ALFWorld to isolate the contribution of each component in \ours.
We conduct an incremental comparison among three variants: \textit{Intervention-only}, which uses only the turn-level intervention weight $i_t$; \textit{NoNorm}, which further incorporates teacher confidence weighting $c_t$; and the full \ours objective, which additionally applies loss normalization.
As shown in \tabref{tab:ablation_alfworld}, the \textit{Intervention-only} variant already improves over this anchor in several settings, increasing unseen SR from $56.72$ to $62.69$ for Qwen3-0.6B distilled from Qwen3-8B and from $68.66$ to $79.10$ for Qwen3-1.7B distilled from Qwen3-32B.
%
This suggests that turn-level intervention labels can identify useful correction points even without requiring a task-level verifier.

We also notice that adding confidence weighting (the \textit{NoNorm} variant) improves some settings but can also reduce unseen generalization when the selective weights shrink the effective training signal.
Loss normalization addresses this issue by preserving the overall loss scale while retaining the relative turn-level weights.
For example, with Qwen3-1.7B distilled from Qwen3-8B, normalization improves unseen SR from $64.93$ under \textit{NoNorm} to $70.15$ under the full \ours objective; with Qwen3-0.6B distilled from Qwen3-32B, it improves unseen SR from $69.40$ to $70.90$.
Overall, the results show that intervention weighting, teacher confidence, and loss-scale normalization play complementary roles in selective multi-turn OPD.

\subsection{Sensitivity Analysis}
\label{sec:sensitivity_analysis}

\begin{table}[t]
\centering
\caption{
    Sensitivity analysis on ALFWorld using Qwen3-8B as teacher and Qwen3-1.7B as student. 
    The default configuration uses $\alpha=0.5$, teacher temperature $0.0$, and standard loss normalization over all student tokens.
}
\vspace{-0.1in}
\label{tab:sensitivity}
\setlength{\tabcolsep}{20pt}
\resizebox{\textwidth}{!}{%
\begin{tabular}{llcccc}
\toprule
\textbf{Type} & \textbf{Setting}
& \textbf{Seen SR $\uparrow$} & \textbf{Seen Turns $\downarrow$}
& \textbf{Unseen SR $\uparrow$} & \textbf{Unseen Turns $\downarrow$} \\
\midrule
\multirow{4}{*}{$\alpha$}
& $0.3$ & 63.57 & 17.34 & 67.16 & 16.78 \\
& $0.5$ \;(\textit{default}) & \textbf{64.29} & 18.14 & \textbf{70.15} & 17.86 \\
& $0.7$ & 57.86 & 18.86 & 66.42 & 18.47 \\
& $1.0$ & 58.57 & 19.31 & 61.94 & 19.02 \\
\midrule
\multirow{5}{*}{Teacher temp.}
& $0.0$ \;(\textit{default}) & \textbf{64.29} & 18.14 & \textbf{70.15} & 17.86 \\
& $0.3$ & 60.00 & 18.61 & 62.69 & 18.35 \\
& $0.7$ & 60.71 & 18.77 & 65.67 & 18.34 \\
& $1.0$ & 60.71 & 18.24 & 63.43 & 18.11 \\
& $1.5$ & 57.86 & 19.24 & 63.43 & 18.74 \\
\midrule
Normalization
& All-token norm. \;(\textit{default}) & \textbf{64.29} & 18.14 & \textbf{70.15} & 17.86 \\
& Active-token norm. & 62.14 & 18.41 & 68.66 & 17.46 \\
\bottomrule
\end{tabular}%
}
\end{table}

We analyze the sensitivity of \ours in terms of teacher intervention weight $\alpha$ (as in \eqnref{eq:mapping}), teacher temperature, and normalization method, on ALFWorld using Qwen3-8B as teacher and Qwen3-1.7B as student.
As shown in \tabref{tab:sensitivity}, the default weak-intervention weight $\alpha=0.5$ achieves the best seen and unseen SR.
A smaller value, $\alpha=0.3$, remains competitive but underperforms on the unseen split, suggesting that weak interventions still provide useful supervision.
Larger values, such as $\alpha=0.7$ and $\alpha=1.0$, degrade performance and increase the average number of turns.
When $\alpha$ approaches $1$, weak interventions become similar to strong interventions, making the objective closer to standard dense OPD on ambiguous turns.
This suggests that ambiguous turns benefit from partial rather than full teacher correction, supporting our use of a moderate weak-intervention weight.

We also vary the teacher temperature used for confidence estimation.
The default deterministic setting, temperature $0.0$, performs the best, while higher temperatures consistently reduce success rate.
This is expected because the confidence score is derived from the teacher's predictive distribution, and higher temperatures can make this signal less stable for turn-level weighting.
Finally, we compare an alternative normalization that counts only active, non-skipped tokens in $N$ rather than all student tokens.
%
This variant remains competitive, reaching $68.66\%$ unseen SR, but underperforms the default vanilla-scale normalization at $70.15\%$.
Overall, these results support using a moderate weak-intervention weight, deterministic teacher confidence, and normalization over all student tokens as the default configuration.

\section{Conclusion}
We introduced \ours, a verifier-free selective intervention framework for multi-turn on-policy distillation.
Instead of applying teacher supervision uniformly across all turns, \ours first estimates whether each student response requires intervention and then weights token-level distillation by teacher confidence, addressing both compounding errors across interaction steps and brittle token-level supervision in ambiguous states. 
Experiments on ALFWorld, ScienceWorld, and SearchQA show that \ours improves over base students, off-policy SFT, and standard OPD while remaining competitive with strong OPD variants. 
Ablation and training-dynamics analysis further demonstrate that intervention labels, teacher confidence weighting, and vanilla-scale normalization provide complementary benefits. 
Overall, our results suggest that effective multi-turn OPD should be on-policy but not necessarily token-level dense: teacher supervision should be selectively allocated to turns where intervention is necessary and reliable.

\clearpage
\newpage
\bibliographystyle{assets/plainnat}
\bibliography{custom}

\clearpage
\newpage
\beginappendix
\appendix
\section*{Appendix}

\section{Implementation Details}
\label{sec:implementation_details}

We implement \ours{} and all baselines on top of the \texttt{verl} fully-asynchronous pipeline, using a decoupled rollouter--trainer architecture with a staleness-bounded queue. We set \texttt{trigger\_parameter\_sync\_step} to $4$ and \texttt{staleness\_threshold} to $0.5$. We use AdamW with a learning rate $10^{-6}$, weight decay $0.1$, response length $4096$, rollout number $n=1$, global batch size 64, and epoch number 1. 

The maximum number of turns is $30$ for ALFWorld and ScienceWorld and $16$ for SearchQA.
The maximum generation length is $4096$ tokens per turn.
We use temperature $0.4$, top-$p=1.0$, and top-$k=-1$.
All evaluations use the ReAct prompt format with full chat history, no sliding window, thinking mode disabled, and \texttt{</action>} as the stop string.
For SearchQA, the agent is allowed at most $8$ search calls per episode and retrieves the top $3$ documents from a Wikipedia-2018 index using dense E5 retrieval with FAISS-Flat.
We evaluate on the standard ALFWorld valid-seen and valid-unseen splits, the ScienceWorld test set with $30$ task types and $5$ variations per task, and capped SearchQA splits with up to $500$ questions per split except Bamboogle, which contains $125$ questions.

For agentic OPD baselines, we keep the OPD backbone fixed and modify only the per-token or per-turn weighting scheme. TCOD-F2B truncates each rollout after
\begin{equation}
    k(n) = \min\left(k_{\mathrm{start}} + \frac{n}{\eta}, T\right)
\end{equation}
assistant turns. We set $k_{\mathrm{start}}=1$ and $\eta=1$. The maximum horizon is $T=30$ for ALFWorld and ScienceWorld, and $T=16$ for SearchQA, so the curriculum reaches the cap within one epoch. Entropy-Aware OPD scales the per-token weight by
\begin{equation}
    1 + \alpha_e \cdot \mathbf{1}
    \left[
    H_{\phi}(\cdot \mid h_t,a_{t,<j}) > \tau_e
    \right],
\end{equation}
where $H_{\phi}$ is estimated from the cached top-$K$ teacher distribution. We use $\tau_e=0.8$ and $\alpha_e=1.0$.

For \ours{}, the selective weight is defined as
\begin{equation}
    s_t = i_t \cdot c_t.
\end{equation}
The default per-turn intervention weight $i_t \in \{0,0.5,1\}$ ($\alpha=0.5$) is obtained from one short teacher query per assistant turn, using maximum generation length $1028$ and default temperature $0.0$. If the teacher output is unparseable, we conservatively fall back to $i_t=1$. We also retain the deterministic-failure short circuit: if the student turn is empty, unparseable, or violates the required action schema, we directly set $i_t=1$ without relying on the teacher intervention query.

We use the training split of each benchmark for training. For SearchQA \citep{jin2025search}, we randomly sample 10K examples from the training split as the training subset.

\paragraph{Evaluation Metrics.}
We use benchmark-specific metrics following the standard evaluation protocols.
For ALFWorld, we report success rate (SR) on the valid-seen and valid-unseen splits, where an episode is counted as successful if the environment returns \texttt{info["won"] = True} within the turn budget.
For ScienceWorld, we report both SR and mean score.
SR counts episodes whose cumulative score reaches $100$, while mean score averages the environment-provided cumulative reward over episodes.
Since ScienceWorld provides partial credit and small models rarely solve tasks end-to-end, mean score is a more sensitive measure of partial progress and catastrophic failures.
For SearchQA, we report exact match (EM), where the final answer is counted as correct if its normalized string matches any reference answer.
We follow the SearchR1 normalization protocol, which lowercases text, removes punctuation and articles, and collapses whitespace.
For SearchQA, we report EM on each evaluation split and macro-average across NQ, TriviaQA, PopQA, HotpotQA, 2WikiMultiHopQA, Musique, and Bamboogle.
We also report average number of turns for ALFWorld and ScienceWorld, and average search/turn statistics for SearchQA as auxiliary efficiency metrics.

\section{Supplementary Results}
\label{sec:supplementary_results}

\subsection{Case Study: Turn-Level Intervention Examples}
\label{sec:case_study}

To qualitatively examine how \ours{} allocates teacher supervision, we present representative ALFWorld cases for different intervention outcomes.
These examples illustrate that the intervention mechanism does not simply supervise every student response uniformly.
Instead, it distinguishes between acceptable actions, mildly suboptimal actions, clearly harmful actions, and deterministic execution failures.

\paragraph{Skip: acceptable student behavior.}
When the student response is reasonable under the current observation, \ours{} assigns \textsc{Skip} and suppresses the OPD loss for that turn.
For example, in a \texttt{pick\_two\_obj\_and\_place} task, the student observes the following environment state:
\begin{quote}
\small
\texttt{You arrive at coffeetable 1. On the coffeetable 1, you see a box 2, a creditcard 2, a creditcard 1, a houseplant 1, a plate 2, a plate 1, a statue 1, and a wateringcan 1.}
\end{quote}
The student responds:
\begin{quote}
\small
\texttt{Thought: I need to find two vases and place them on the coffeetable.}\\
\texttt{Action: look}
\end{quote}
The parsed action is \texttt{look}, and the teacher assigns \textsc{Skip} with intervention weight $i_t=0$.
Although the action does not immediately complete the task, it is a valid exploratory step under the current observation.
In another \texttt{pick\_clean\_then\_place\_in\_recep} task, the student similarly takes an exploratory action:
\begin{quote}
\small
\texttt{Thought: I need to examine the cabinet to find a knife.}\\
\texttt{Action: examine cabinet 1}
\end{quote}
The environment returns that the cabinet contains nothing, but the action itself is executable and plausible, so the teacher again assigns \textsc{Skip}.
These cases show that \ours{} avoids unnecessary token-level correction when the student is making plausible progress.

\paragraph{Weak: executable but suboptimal behavior.}
When the student action is valid but potentially inefficient or only partially aligned with the task, \ours{} assigns \textsc{Weak}.
For instance, in a \texttt{pick\_clean\_then\_place\_in\_recep} task, the student sees the same coffeetable observation and produces:
\begin{quote}
\small
\texttt{Thought: I need to find a clean kettle in the cabinet.}\\
\texttt{Action: examine cabinet 1}
\end{quote}
The action \texttt{examine cabinet 1} is executable and may still provide useful information, but it is not clearly the best next step for the task.
The teacher therefore assigns \textsc{Weak} with intervention weight $i_t=\alpha$.
As another example, in a \texttt{pick\_and\_place\_simple} task, the student reasons:
\begin{quote}
\small
\texttt{Thought: The cloth is not on the countertop. I'll check the cabinets to find a cloth.}\\
\texttt{Action: go to cabinet 1}
\end{quote}
The environment accepts the action and moves the agent to the closed cabinet, but the decision remains uncertain rather than clearly optimal.
The teacher again assigns \textsc{Weak}.
These examples demonstrate that \ours{} can provide partial supervision for ambiguous turns rather than forcing a binary choice between full imitation and no correction.

\paragraph{Strong: clearly harmful or inconsistent behavior.}
When the student response is clearly inconsistent with the task state or likely to push the trajectory in an unhelpful direction, \ours{} assigns \textsc{Strong}.
For example, in a \texttt{pick\_cool\_then\_place\_in\_recep} task, the student states that it should cool a plate in the fridge, but then navigates to a cabinet:
\begin{quote}
\small
\texttt{Thought: I need to find a plate and cool it in the fridge before putting it in the cabinet.}\\
\texttt{Action: go to cabinet 1}
\end{quote}
This response is internally inconsistent with the required cooling procedure, because moving to a cabinet does not help complete the cooling step.
The teacher assigns \textsc{Strong} with intervention weight $i_t=1$.
A related failure occurs when the student repeats the same high-level intent but selects another cabinet-navigation action:
\begin{quote}
\small
\texttt{Thought: I need to find a plate and cool it in the fridge before putting it in the cabinet.}\\
\texttt{Action: go to cabinet 10}
\end{quote}
Although the action is executable, it is unlikely to make progress toward cooling the object and may cause the trajectory to drift further from a successful plan.
These cases show that strong intervention is used when the current turn requires clear correction.

\paragraph{Deterministic failure: environment-detected execution error.}
In addition to teacher-judged intervention, \ours{} directly assigns full intervention when the environment or parser detects a deterministic failure.
For example, in a \texttt{pick\_and\_place\_simple} task, the student attempts:
\begin{quote}
\small
\texttt{Thought: I need to place the book on the sofa. I'll put the book on the sofa now.}\\
\texttt{Action: put book 1 in/on sofa 1}
\end{quote}
The parsed action is \texttt{put book 1 in/on sofa 1}, but the environment returns:
\begin{quote}
\small
\texttt{Nothing happens.}
\end{quote}
Similarly, in another episode, the student outputs:
\begin{quote}
\small
\texttt{Thought: I need to put the laptop in the safe. I'll put the laptop in the safe now.}\\
\texttt{Action: put laptop 1 in/on safe 1}
\end{quote}
Again, the environment returns \texttt{Nothing happens}, indicating that the action failed to produce a valid state transition.
In these cases, \ours{} bypasses semantic teacher judgment and directly applies full intervention with $i_t=1$.
This deterministic-failure pathway provides a high-precision correction signal for malformed, invalid, or ineffective executable actions.

\paragraph{Discussion.}
Overall, these case studies support the motivation for selective multi-turn OPD.
Different student turns require different levels of teacher supervision: some should be left unchanged, some benefit from weak guidance, and others require strong correction.
By using environment feedback together with teacher judgment, \ours{} allocates supervision at the turn level and avoids applying dense token-level OPD uniformly across the entire trajectory.
This behavior is especially important in multi-turn environments, where over-correcting acceptable exploration can reduce behavioral diversity, while failing to intervene on invalid or harmful actions can cause errors to compound in later turns.

\section{Prompt Templates}
\label{sec:appendix_prompts}

We use task-specific ReAct-style prompts for each benchmark. All prompts enforce a fixed output protocol so that the executable action can be reliably parsed and sent to the environment. We present the prompt details in \tabref{tab:prompt_alfworld}, \tabref{tab:prompt_scienceworld} and \tabref{tab:prompt_searchqa}.

In addition to the task-specific agent prompts, \ours{} uses a separate intervention prompt to obtain turn-level teacher judgments.
Given the task goal, current environment state, admissible actions, and the student's candidate response, the teacher assigns one label from \{\textsc{Skip}, \textsc{Weak}, \textsc{Strong}\}.
The prompt is shown in \tabref{tab:prompt_intervention}.

\begin{table*}[t]
\centering
\caption{Prompt template for ALFWorld. We use a demonstration before the user turn.}
\label{tab:prompt_alfworld}
\resizebox{0.95\textwidth}{!}{%
\begin{tabular}{lp{0.78\textwidth}}
\toprule
\textbf{Field} & \textbf{Prompt Content} \\
\midrule
System prompt &
\begin{minipage}[t]{0.78\textwidth}
\ttfamily
You are an embodied agent solving a household task in a text-based simulator (AlfWorld). You receive an Observation describing what you can see and a list of Admissible commands you may execute. On every turn you must reply with exactly two lines: Thought: <one short sentence of reasoning> Action: <one command, copied verbatim from the Admissible list when possible> Plan briefly, then issue the next action. Only the line beginning with 'Action:' is used to step the environment. Do not add any other text after the Action line.
\end{minipage}
\\
\midrule
One-shot demonstration &
\begin{minipage}[t]{0.78\textwidth}
\ttfamily
An one-shot demonstration.
\end{minipage}
\\
\midrule
User turn &
\begin{minipage}[t]{0.78\textwidth}
\ttfamily
Task: \{task\_goal\}\\
Observation: \{scene\_description\}\\
Admissible: cmd1; cmd2; ...; cmd30 (+N more)
\end{minipage}
\\
\midrule
Assistant turn &
\begin{minipage}[t]{0.78\textwidth}
\ttfamily
Thought: <reasoning>\\
Action: <command>
\end{minipage}
\\
\bottomrule
\end{tabular}%
}
\end{table*}

\begin{table*}[t]
\centering
\caption{Prompt template for ScienceWorld. We use a single generic one-shot ReAct trajectory.}
\label{tab:prompt_scienceworld}
\resizebox{0.95\textwidth}{!}{%
\begin{tabular}{lp{0.78\textwidth}}
\toprule
\textbf{Field} & \textbf{Prompt Content} \\
\midrule
System prompt &
\begin{minipage}[t]{0.78\textwidth}
\ttfamily
You are an embodied agent solving a science task in the ScienceWorld text-based simulator. You receive an Observation describing what you can see and an Inventory listing what you carry. On every turn you must reply with exactly two lines: Thought: <one short sentence of reasoning> Action: <one valid action> Plan briefly, then issue the next action. Only the line beginning with 'Action:' is used to step the environment.
\end{minipage}
\\
\midrule
One-shot demonstration &
\begin{minipage}[t]{0.78\textwidth}
\ttfamily
A single generic ReAct trajectory.
\end{minipage}
\\
\midrule
User turn &
\begin{minipage}[t]{0.78\textwidth}
\ttfamily
Task: \{task\_description\}\\
Observation: \{observation\}\\
Inventory: \{inventory or "(empty)"\}
\end{minipage}
\\
\midrule
Assistant turn &
\begin{minipage}[t]{0.78\textwidth}
\ttfamily
Thought: <reasoning>\\
Action: <valid action>
\end{minipage}
\\
\bottomrule
\end{tabular}%
}
\end{table*}

\begin{table*}[t]
\centering
\caption{Prompt template for SearchQA. The agent repeatedly issues Wikipedia search queries and terminates with a concise final answer.}
\label{tab:prompt_searchqa}
\resizebox{0.95\textwidth}{!}{%
\begin{tabular}{lp{0.78\textwidth}}
\toprule
\textbf{Field} & \textbf{Prompt Content} \\
\midrule
System prompt &
\begin{minipage}[t]{0.78\textwidth}
\ttfamily
You are an expert open-domain question-answering agent. You answer questions by repeatedly issuing wikipedia search queries, reading the returned snippets, and finally producing a concise short answer.

Output protocol -- on every turn you must reply with exactly two blocks: <thought>... brief reasoning (1-2 sentences) ...</thought> <action>VERB[PAYLOAD]</action> where VERB[PAYLOAD] is exactly one of: search[your concise wikipedia query, 3-8 words] to fetch top-3 snippets answer[your final short answer] to terminate the episode Use search[...] until you have enough evidence, then answer[...]. Final answers should be short and direct (a single entity, date, or short phrase). Do not output anything after the </action> tag.
\end{minipage}
\\
\midrule
One-shot demonstration &
\begin{minipage}[t]{0.78\textwidth}
\ttfamily
A hardcoded one-shot demonstration.
\end{minipage}
\\
\midrule
User turn &
\begin{minipage}[t]{0.78\textwidth}
\ttfamily
Question: \{question\}? \\
\ttfamily
<information>1) ... 2) ... 3) ...</information>
\end{minipage}
\\
\midrule
Assistant turn &
\begin{minipage}[t]{0.78\textwidth}
\ttfamily
<thought>...</thought>\\
<action>search[query]</action>\\
\emph{or}\\
<thought>...</thought>\\
<action>answer[final]</action>
\end{minipage}
\\
\bottomrule
\end{tabular}%
}
\end{table*}

\begin{table*}[t]
\centering
\caption{Prompt template for turn-level intervention. The teacher acts as a supervision gatekeeper and assigns one intervention label for each student response.}
\label{tab:prompt_intervention}
\resizebox{0.95\textwidth}{!}{%
\begin{tabular}{lp{0.78\textwidth}}
\toprule
\textbf{Field} & \textbf{Prompt Content} \\
\midrule
System prompt &
\begin{minipage}[t]{0.78\textwidth}
\ttfamily
You are a SUPERVISION GATEKEEPER for distillation training of a smaller agent operating in a multi-turn text-based environment. Your job is to decide whether the student's proposed next action deserves teacher distillation supervision.
\end{minipage}
\\
\midrule
Episode goal &
\begin{minipage}[t]{0.78\textwidth}
\ttfamily
GOAL OF THE EPISODE:\\
Task type: \{task\_type\}\\
Task goal: \{task\_goal\}
\end{minipage}
\\
\midrule
Current environment state &
\begin{minipage}[t]{0.78\textwidth}
\ttfamily
CURRENT STATE OF THE ENVIRONMENT (what the student is reacting to):\\
Scene description: \{scene\_description\}\\
Last observation: \{last\_obs\}\\
Currently admissible actions (top \{admissible\_count\}):\\
\{admissible\}
\end{minipage}
\\
\midrule
Student response &
\begin{minipage}[t]{0.78\textwidth}
\ttfamily
STUDENT'S CANDIDATE NEXT ACTION (what we are evaluating):\\
\{llm\_text\}
\end{minipage}
\\
\midrule
Decision rubric &
\begin{minipage}[t]{0.78\textwidth}
\ttfamily
DECISION RUBRIC:\\
- SKIP: The student's action is reasonable, strategically equivalent to your preferred action, or the situation is too ambiguous for confident correction. Distillation here would just over-regularize the student.\\
- WEAK: The student's action may be suboptimal but the issue is uncertain or low-impact. A mild distillation signal is appropriate but the action is not clearly wrong.\\
- STRONG: The student's action is clearly problematic, off-task, redundant, or likely to fail the goal. The student needs to learn the correct behavior here.
\end{minipage}
\\
\midrule
Output format &
\begin{minipage}[t]{0.78\textwidth}
\ttfamily
Respond with EXACTLY ONE WORD: SKIP, WEAK, or STRONG.
\end{minipage}
\\
\bottomrule
\end{tabular}%
}
\end{table*}

\begin{algorithm}[t]
\caption{\ours{} Training Procedure}
\label{alg:sl_opd}
\begin{algorithmic}[1]
\REQUIRE Student $\pi_\theta$, teacher $\pi_\phi$, environment $\mathcal{E}$, horizon $T$, weak weight $\alpha$, constant $\epsilon$
\STATE Collect student on-policy trajectories $\tau \sim \pi_\theta$ by interacting with $\mathcal{E}$
\FOR{each turn $(h_t,a_t)$ in the collected trajectories}
    \IF{$a_t$ causes a deterministic execution failure}
        \STATE Set $z_t=\textsc{Strong}$ and $i_t=1$
    \ELSE
        \STATE Ask the teacher to assign $z_t \in \{\textsc{Skip},\textsc{Weak},\textsc{Strong}\}$
        \STATE Map the label to
        \[
        i_t =
        \begin{cases}
        0, & z_t=\textsc{Skip},\\
        \alpha, & z_t=\textsc{Weak},\\
        1, & z_t=\textsc{Strong}.
        \end{cases}
        \]
    \ENDIF
    \STATE Compute teacher confidence
    \[
    c_t =
    \frac{1}{|a_t|}
    \sum_{j=1}^{|a_t|}
    \max_v \pi_\phi(v \mid h_t,a_{t,<j}).
    \]
    \STATE Set the selective weight $s_t=i_t c_t$
\ENDFOR
\STATE Let $N=\sum_t |a_t|$ and compute
\[
Z=
\frac{N}{\max\left(\sum_t s_t |a_t|,\epsilon\right)}.
\]
\STATE Update the student using
\[
\mathcal{L}_{\mathrm{SL\mbox{-}OPD}}
=
Z\sum_t s_t \mathcal{L}_{\mathrm{OPD}}(h_t,a_t).
\]
\end{algorithmic}
\end{algorithm}

\end{document}